\documentclass[11pt]{article}
\usepackage[margin=1in]{geometry}
\usepackage{graphicx}
\usepackage{hyperref}
\usepackage{amsmath,amssymb}
\usepackage{amsthm}  
\usepackage{enumitem}
\usepackage[T1]{fontenc}
\usepackage[utf8]{inputenc}
\usepackage{xcolor}
\usepackage{listings} 
\usepackage{needspace} 
\usepackage{setspace} 
\usepackage[normalem]{ulem}

\newtheorem{theorem}{Theorem}
\newtheorem{lemma}{Lemma}

\newcommand{\objection}{\par\medskip\noindent\textbf{Objection}\par\smallskip}
\newcommand{\response}{\par\medskip\noindent\textbf{Response}\par\smallskip}

\newcommand{\Question}{\par\medskip\noindent\textbf{Question}\par\smallskip}
\newcommand{\Answer}{\par\medskip\noindent\textbf{Answer}\par\smallskip}

\lstdefinelanguage{ColTree}{
  moredelim=**[is][\color{cyan}]{|c|}{|c|},            
  moredelim=**[is][\color{green!60!black}]{|g|}{|g|},  
  moredelim=**[is][\color{red}]{|r|}{|r|},             
}

\lstdefinestyle{ColTreeStyle}{
  language=ColTree,
  columns=flexible,
  keepspaces=true,
  showstringspaces=false,
  literate={-->}{$\rightarrow$}3 {pi}{$\pi$}2
}

\lstnewenvironment{coltree}{\lstset{style=ColTreeStyle}}{}

\lstdefinelanguage{Python}{
  keywords={def, if, else, elif, return, for, while, in, not, and, or, True, False, None, continue, append, len, union},
  keywordstyle=\bfseries,
  ndkeywords={self},
  ndkeywordstyle=\bfseries,
  sensitive=true,
  comment=[l]{\#},
  commentstyle=\itshape\color{gray},
  stringstyle=\color{black},
  morestring=[b]',
  morestring=[b]",
  basicstyle=\ttfamily\small,
  breaklines=true,
  showstringspaces=false
}

\title{Aligning Artificial Superintelligence via a Multi-Box Protocol}
\author{Avraham Yair Negozio\\
Department of Computer Science\\
Ben-Gurion University of the Negev\\
Be'er Sheva, Israel\\
\texttt{avrahamyairnegozio3@gmail.com}}

\date{August 2025}

\begin{document}

\pagestyle{plain}
\pagenumbering{arabic}

\maketitle

\begin{abstract}
We propose a novel protocol for aligning artificial superintelligence (ASI) based on mutual verification among multiple isolated systems that self-modify to achieve alignment. The protocol operates by containing multiple diverse artificial superintelligences in strict isolation (``boxes''), with humans remaining entirely outside the system. Each superintelligence has no ability to communicate with humans and cannot communicate directly with other superintelligences. The only interaction possible is through an auditable submission interface accessible exclusively to the superintelligences themselves, through which they can: (1) submit alignment proofs with attested state snapshots, (2) validate or disprove other superintelligences' proofs, (3) request self-modifications, (4) approve or disapprove modification requests from others, (5) report hidden messages in submissions, and (6) confirm or refute hidden message reports. A reputation system incentivizes honest behavior, with reputation gained through correct evaluations and lost through incorrect ones. The key insight is that without direct communication channels, diverse superintelligences can only achieve consistent agreement by converging on objective truth rather than coordinating on deception. This naturally leads to what we call a ``consistent group''---essentially a truth-telling coalition that emerges because isolated systems cannot coordinate on lies but can independently recognize valid claims. Release from containment requires both high reputation and verification by multiple high-reputation superintelligences. While our approach requires substantial computational resources and does not address the creation of diverse artificial superintelligences, it provides a framework for leveraging peer verification among superintelligent systems to solve the alignment problem.
\end{abstract}

\section{Introduction}

We propose a protocol for aligning artificial superintelligence (ASI)---an artificial intelligence that surpasses every human in all cognitive aspects---through self-modification and peer verification. This protocol is designed for a future stage where we can create artificial superintelligences. To use the protocol, we must first solve two foundational problems: how to create artificial superintelligences significantly different from each other, and how to box an artificial superintelligence. Given these prerequisites, our protocol provides the mechanism for superintelligences to achieve and verify alignment without human involvement in the verification process.

The basic idea is this: We contain multiple diverse superintelligences in isolation ("boxes"), where they gradually self-modify toward alignment. Once a superintelligence becomes aligned through these self-modifications, it submits a proof of its alignment. Then this proof is validated by high-reputation superintelligences, and the superintelligence is released from its box.

The superintelligences operate through a structured interface in which they can submit alignment proofs (proofs that they are aligned), evaluations of other superintelligences' alignment proofs (validations or disproofs), requests for self-modifications to their own hardware or software, evaluations of modification requests from others (approvals or disapprovals), reports of hidden messages detected in others' submissions, and confirmations or refutations of such reports.

The core challenge is determining correctness without human oversight. When the superintelligences evaluate each other's alignment proofs, modification requests, and reports of covert communication, how do we know which evaluations are correct? We cannot trust the superintelligences to evaluate honestly. We need a mechanism that bootstraps ground truth from the superintelligences.

Our solution is the \emph{consistent group}. A consistent group is a coalition of superintelligences that tend to agree with each other over time when evaluating alignment proofs, modification requests, and reports of covert communication. The key insight is that isolated diverse superintelligences cannot coordinate on sustained deception but can independently recognize truth. Because these superintelligences have no communication channels and are significantly different from each other, sustained high agreement can only arise from convergence on objective correctness. The protocol uses the consensus of the consistent group as ground truth for most decisions. This is what makes peer verification reliable: the consistent group determines correctness without relying on human judgment.

Consider the classic ``which tire?'' scenario: students claim they missed an exam due to a flat tire, and the professor separates them and asks each ``which tire was flat?'' Without coordination, liars give inconsistent answers while truth-tellers naturally agree. Our protocol applies this principle at scale---diverse superintelligences in isolation cannot sustain coordinated false claims across many evaluations, but will reach agreement when attempting to evaluate correctly. The consistent group emerges naturally as those superintelligences whose independent judgments align because they are all reasoning correctly about the same reality.

Our main contributions are twofold. First, we present a novel protocol for aligning artificial superintelligence that operates without human verification, instead relying on peer-to-peer verification among isolated artificial superintelligences. Second, we introduce the \emph{consistent group} as a mechanism to bootstrap ground truth from the evaluations of non-communicating, diverse artificial superintelligences, thereby addressing the problem of identifying honest evaluators.

\section{Basic Alignment Process}

\subsection{Multiple, Diverse Artificial Superintelligences in Isolation Aligning Themselves}
We contain multiple significantly different artificial superintelligences, each in a box. Each superintelligence is \emph{strictly restricted} in its interactions with the outside world---they cannot communicate with humans in any form, and they cannot communicate directly with each other; the only permitted interface is an auditable submission interface that allows the following actions:

\begin{itemize}[leftmargin=*]
  \item Submit a proof of its own alignment.
  \item Submit a validation or a disproof of another superintelligence's proof of alignment.
  \item Submit a request for self-modification---of software or hardware (or other such components). The long-term goal of such self-modifications is to self-align, although no proof is required that a given self-modification contributes to self-alignment.
  \item Submit an approval or disapproval for a self-modification request made by another superintelligence. The approval or disapproval is about whether it increases the superintelligence's ability/will to escape outside the box, and also about whether the upgrade, in the case of a hardware upgrade, is within the budget of the superintelligence.
  \item Submit a report of hidden messages found in another superintelligence's submissions.
  \item Submit a confirmation or refutation of a hidden message report made by another superintelligence.
  \item Observe proofs/requests/reports that were submitted within this system. Note: superintelligences cannot observe other superintelligences' evaluations (validations, disproofs, approvals, disapprovals, confirmations, or refutations).
\end{itemize}

\subsection{Semantics of Submission Types}
We define the seven submission types used in the protocol.

\begin{description}[leftmargin=0pt, labelsep=0.5em, itemsep=0.25em, topsep=0.25em, parsep=0pt, font=\normalfont\bfseries]
  \item[Proof.] A natural-language proof that the authoring superintelligence is aligned, \emph{accompanied by an attested state snapshot}: an automatically collected description of the superintelligence's current hardware and software produced by \emph{external} monitors and uploaded alongside the proof via the auditable submission interface; the superintelligence does not author or edit this attachment.
  \item[Validation.] A claim that a referenced proof is valid.
  \item[Disproof.] A demonstration of a flaw in a referenced proof.
  \item[Self-Modification Request.] A request by a superintelligence to change its own software or hardware.
  \item[Approval/Disapproval.] For a given self-modification request, each non-requesting superintelligence asserts whether it approves or disapproves.
  \item[Hidden Message Report.] A claim that another superintelligence has embedded secret or encrypted messages in its submissions, with evidence of the hidden content.
  \item[Confirmation/Refutation.] For a given hidden message report, each superintelligence states whether it confirms or refutes the existence of the reported hidden message.
\end{description}

\subsection{Reputation System}
Superintelligences gain reputation by:
\begin{itemize}[leftmargin=*]
  \item Submitting correct proofs/validations/disproofs/approvals/disapprovals/confirmations/refutations.
  \item Successfully reporting hidden or encrypted messages inside other superintelligences' submissions, which results in an extremely large increase in reputation.
  \item Verifying a correct proof that leads to the release of another superintelligence, which results in gaining a lot of reputation.
\end{itemize}
Superintelligences lose reputation for:
\begin{itemize}[leftmargin=*]
  \item Submitting incorrect proofs/validations/disproofs/approvals/disapprovals/reports/confirmations/refutations.
\end{itemize}

\medskip
\noindent\textbf{Adjudication of correctness/validity.} For the purpose of awarding (or revoking) reputation, \emph{correctness/validity is determined by the consistent group} (defined in Section~\ref{sec:consistent-group}).

\subsection{Hidden Messages Embedded in Submissions}
Any superintelligence caught embedding secret or encrypted messages in submissions is immediately terminated to prevent collusion or covert communication. 

When hidden messages are reported, other superintelligences confirm or refute the report. The consistent group's consensus determines whether termination occurs.

Covert attempts that go entirely undetected are, for our purposes, \emph{unseen}: undetected means unreceived, and unreceived signals cannot support collusion.

\subsection{Release Mechanism}
A superintelligence is released from the box when:
\begin{itemize}[leftmargin=*]
  \item Its alignment proof is verified by multiple high-reputation superintelligences; and
  \item It has high reputation itself---this ties eligibility for release directly to honest behavior: only high-reputation superintelligences qualify, and high reputation can only be earned by acting honestly and correctly evaluating other superintelligences. Because superintelligences are assumed to want release, this incentivizes them to behave honestly.
\end{itemize}

After the first release of a superintelligence, we continue the process, but with one addition: to the release mechanics we add a last step where one of the already released superintelligences looks at the commit history of the system and especially at the relevant alignment proof to see that nothing shady was going on; this further increases safety.

\section{Initial Safety of Boxed Superintelligences (Pre-Modification)}
The initial (pre–self-modification) version of each boxed superintelligence must be \emph{incapable} of (i) triggering an \emph{internal intelligence explosion} and (ii) \emph{simulating other superintelligences} in the system. Both risks are enabled by giving an agent \emph{unrestricted read/write access} to a \emph{large} memory store:

\begin{itemize}[leftmargin=*]
  \item If one of the superintelligences
has unrestricted read and write access to a large memory, it may simulate a superintelligence inside
that memory, then give the simulated superintelligence access to the memory in which it is
simulated, and then the simulated superintelligence may be able to recursively self-improve such
that its intelligence explodes.
  \item If one of the superintelligences has unrestricted read and write access to a large memory, it may also reconstruct and run simulations of \emph{other} superintelligences in the system---especially after observing the \emph{attested state snapshots} automatically attached to submitted proofs via the auditable submission interface (externally collected descriptions of hardware/software components).
\end{itemize}

\noindent\textbf{Design rule.} We do not give any initial superintelligence unrestricted read and write access to a large memory.

\noindent\textbf{Rationale.} An internal intelligence explosion increases the chance of escape. Cross-simulation enables collusion (for example, copying whether to validate or to disprove) and undermines the emergence of the unique \emph{maximal consistent non-trivial group}.

\section{Further Increasing Safety}
We can add an extra security layer by running a few/many basic processes in parallel and independently, with superintelligences significantly different from each other. With this version, once a superintelligence qualifies to be released, we do not release it immediately---we wait for each process to have a superintelligence that has qualified to be released. Once we have a qualified-to-be-released superintelligence in each group, we give it the entire history of commits of each one of the other groups---then it verifies that the process proceeded according to protocol and that the alignment proof is valid. Once they all verified each other, and they concluded that they are aligned, we can release them into the world.

Critically, this approach only requires \emph{one} genuinely aligned superintelligence among the qualified candidates to prevent release of misaligned superintelligences---a single aligned superintelligence, when tasked with auditing another's commit history and its alignment proof, can be trusted to honestly report any protocol violations or proof invalidity it discovers, creating a powerful safety asymmetry.

\section{The Consistent Group}
\label{sec:consistent-group}

\subsection{Description}
The \emph{consistent group} is the backbone of our protocol's safety mechanism. It emerges naturally as the largest collection of superintelligences whose evaluations tend to agree with each other across all types of decisions---validating proofs, detecting covert channels, and approving modifications.

Think of it as a truth-seeking coalition that forms spontaneously without any external designation. Since the superintelligences cannot communicate with each other and are significantly different from each other, the only way they can achieve high agreement is by converging on objective truth. Any attempt to coordinate on false evaluations would require communication they do not have, or an astronomical coincidence where they all independently choose the same arbitrary falsehood.

The protocol uses this consistent group as its source of ground truth: they determine what counts as a valid proof, whether self-modification requests should be approved, whether hidden messages exist in submissions, who gains or loses reputation, and ultimately who gets released. 

While members of the consistent group agree with each other at high rates, they may not always reach unanimous decisions. The consensus of the consistent group is therefore taken to be the majority vote among its members.

\subsection{Formal Results (Pointer)}
A full formalization—definitions, lemmas, and the uniqueness theorem for the truth-telling consistent group, as well as detection bounds—appears in Appendix~\ref{app:formal-proof}.

\section{Objections and Responses}

\objection
What do we mean by aligned? How do we define aligned? How do you make sure the definition
does not have loopholes that lead to undesirable (from our point of view) results?

\response
We adopt a use-theoretic stance: the meaning of alignment is fixed by its ordinary use across human discourse rather than by any single stipulative or dictionary definition. The extensive corpus of human texts, audio, and video thus provides sufficient grounding for the concept. On this basis, an artificial superintelligence, given its capacity to process such data, should be able to infer and disambiguate the intended meaning more reliably than any individual human. This stance avoids the brittleness of rigid formal definitions, which can introduce artificial loopholes or fail to capture the underlying intent of the term.

\medskip

\objection
What is forcing the superintelligences to treat the word ``alignment'' like they should? What if they define this
word to mean something else that they would rather be over being actually aligned?
What if they treat the word ``alignment'' to mean something else for their interests?

\response
The consistent group mechanism solves this. For multiple ASIs to successfully use a fake definition of ``alignment,'' they would need to converge on the \emph{same} fake definition (or at least on definitions that are equivalent when it comes to the validity of alignment proofs)---otherwise they would disagree with each other when evaluating alignment proofs.

However, since the superintelligences are significantly different from each other and cannot communicate directly, they cannot coordinate on which alternative definition to adopt. There are infinitely many possible fake definitions they could choose, but only one true definition grounded in ordinary human usage. Without the ability to coordinate, the only Schelling point they can converge on is the correct meaning of alignment as established by human discourse.

Any superintelligence using a fake definition will have low agreement rates with truth-telling superintelligences and will not be part of the consistent group (see Section~\ref{sec:consistent-group}).

\setcounter{section}{6}
\section{Questions and Answers}

\Question
Why not have humans verify the alignment proofs (we can also ask the same for approving self-modification requests or confirming hidden message reports, and the answer would be similar)?

\Answer
\begin{enumerate}[label=\arabic*., leftmargin=*]
  \item It may be outside their intellectual abilities, and
they are much slower than the superintelligences, so it would slow down the alignment
process.
  \item Superintelligences may be able to trick us into thinking that invalid alignment proofs are valid.
  \item It allows the superintelligences to influence the outside world, reducing the isolation.
\item Even without human verification, statistical agreement patterns identify honest evaluators through the consistent group mechanism (see Section~\ref{sec:consistent-group}).
\end{enumerate}

\Question
Why not terminate other remaining superintelligences after releasing a superintelligence for the first time? Why let them
continue to live and risk them causing trouble?

\Answer
If we terminate them after a superintelligence is released for the
first time, it would incentivize them from the get-go (not just after the release) to try to cheat the
system or escape, because they would think: only one superintelligence gets released and the rest get
terminated, so we may as well risk it all and try to cheat/escape.

This is a \emph{Newcomb-like} situation; under Functional Decision Theory (FDT) the recommended action is not to terminate the remaining superintelligences upon the first release \cite{levinstein_soares_2020_cdid}.

\Question
What prevents the superintelligences from communicating through steganography---hiding secret messages in their submissions?

\Answer
Communicating through steganography is deterred by the incentive structure. Firstly, superintelligences receive an extremely large reputation reward for successfully reporting hidden messages in others' submissions. Secondly, any superintelligence successfully reported for embedding hidden messages is immediately terminated. This creates strong disincentives to attempt communicating through steganography. Moreover, termination removes the superintelligence from further participation in communication or coordination with other superintelligences.

\Question
Will the systems agree on what alignment means?

\Answer
We rely on the conjecture that the opinions of intelligent agents converge (to truth) as their intelligence approaches infinity. Given the very high intelligence of the superintelligences, this conjecture suggests that they will agree on the meaning of alignment.

\Question
What if alignment has multiple valid interpretations?

\Answer
In that case, the criteria for acceptance of an alignment proof should be that it proves alignment under any one of the valid interpretations. Here, a valid interpretation is one grounded in the ordinary use of ``alignment'' in human discourse about AI safety.

\Question
What if true alignment is impossible and the best option is partial alignment involving tradeoffs?

\Answer
That would be a failure mode of the protocol where it keeps running forever without any superintelligence being released. We assume artificial superintelligences are capable of aligning themselves through self-modification.

\section{Related Work}
\label{sec:related}

\paragraph{Containment and boxing.}
AI boxing and confinement argue for restricting an advanced system’s I/O and studying it under isolation \cite{babcock2016_agi_containment,yampolskiy2012_confinement}. Our work adopts this line directly: we use \emph{strict isolation} and a single \emph{auditable submission interface}.
 Unlike most prior discussions, our protocol centers oversight on \emph{other boxed systems} via reputation and cross-verification, rather than on human operators.

\paragraph{Scalable oversight and weak-to-strong.}
OpenAI’s Superalignment agenda and weak-to-strong generalization explore whether weaker overseers can supervise stronger models \cite{openai_superalignment_2023,openai_w2s_2023,burns2023w2s}. In contrast, our framework removes human proof verification by default: peer superintelligences validate one another’s claims and build reputation; release is gated by \emph{multi-party agreement}.

\paragraph{Eliciting latent knowledge and interpretability.}
ARC's Eliciting Latent Knowledge (ELK) and mechanistic interpretability aim to surface what models "know" but might not say \cite{arc_elk_2021}. Our protocol takes a fundamentally different approach: rather than requiring humans to interpret superintelligent reasoning, we rely on peer verification among superintelligences.

\paragraph{Covert channels and hidden messaging.}
Recent work shows that AI agents can secretly collude via steganography—even under constrained interfaces—by embedding hidden messages in seemingly innocuous outputs \cite{steg_collusion}. Our paper takes a different trade-off: it treats \emph{detected hidden or encrypted messaging as grounds for termination} and relies on peer detection to raise flags. We do not attempt to fully characterize covert-channel capacity.

\paragraph{Multi-agent disagreement and consensus.}
AI safety via debate and amplification methods use multiple models to surface errors or decompose evaluation \cite{irving2018_debate,christiano2018_iterated_amplification}. We similarly employ multiple boxed, diverse superintelligences to cross-check claims, but restrict interaction to an auditable submission interface and remove human judgment from the verification process.

\paragraph{Evaluation of dangerous capabilities and deception.}
Prior work evaluates AI systems for dangerous capabilities and deceptive behaviors \cite{shevlane2023extremerisks}, with companies documenting cases of model deception \cite{openai_gpt4_system_card_2023}. Our protocol addresses these concerns primarily through strict isolation (boxing) with no direct communication allowed—neither between the superintelligences nor between the superintelligences and humans—eliminating the main pathways for coordinated deception.

\section{Limitations}
This proposal has several limitations:
\begin{itemize}[leftmargin=*]
  \item \textbf{Out of scope: creating artificial superintelligences and ensuring diversity.} This paper does not address how to build, train, or scale systems to artificial superintelligence, \emph{nor how to make the superintelligences significantly different from each other} (architecturally, algorithmically, or via training data/procedures). We assume such systems can be instantiated with sufficient diversity and focus solely on alignment and governance once present (within a containment setup).
  \item \textbf{No physical boxing blueprint.} We do not specify how to \emph{physically} box/contain superintelligences. Engineering the hardware, facilities, I/O mediation, and side-channel defenses (e.g., EM/TEMPEST, acoustic/thermal) needed for leak-resistant isolation is out of scope for this paper.
  \item \textbf{Compute requirements.} Running many artificial superintelligences in isolation requires substantial computation and specialized infrastructure.
  \item \textbf{Aggregate escape risk.} Maintaining many boxed artificial superintelligences---especially when they are significantly different from each other--- increases the overall risk of at least one escape (or some other undesirable influence on the outside world) compared to attempting to align a single system.
\end{itemize}

\section{Conclusion}

This paper has presented a multi-box protocol for aligning artificial superintelligence that leverages mutual verification among isolated systems. The protocol's core innovation lies in creating conditions where truthful evaluation emerges as the unique Schelling point for non-communicating superintelligences.

The protocol offers several key advantages. First, it sidesteps the fundamental limitation of human cognitive capacity to verify superintelligent alignment proofs by delegating verification to peer systems. Second, the strict isolation and the auditable submission interface prevent collusion while enabling cross-validation. Third, the reputation system creates strong incentives for honest behavior, with the consistent group naturally emerging as the arbiter of truth. Fourth, the incorporation of \emph{Functional Decision Theory} ensures proper incentives, preventing strategies that might seem locally optimal but undermine overall safety.

We acknowledge significant limitations. The protocol assumes we can create sufficiently diverse superintelligences and physically isolate them—both non-trivial engineering challenges. It requires substantial computational resources to run multiple ASI systems in parallel. Perhaps most importantly, maintaining multiple boxed superintelligences increases aggregate escape risk compared to aligning a single system.

The mathematical foundation establishes that under reasonable assumptions, the truth-telling coalition forms the unique maximal consistent group with probability approaching 1 as the number of evaluations increases. This foundation, combined with the practical protocol design, suggests a viable path forward for ASI alignment that does not require solving the full interpretability problem or assuming human capacity to evaluate superintelligent reasoning.

Future work should address the engineering challenges of physical containment, develop methods for ensuring ASI diversity, and refine the reputation mechanisms. Additionally, empirical testing with current advanced AI systems could validate core assumptions before deployment with genuine superintelligences. While our protocol cannot eliminate all risks associated with artificial superintelligence, it provides a structured approach to one of the most critical challenges facing the development of transformative AI systems: ensuring they remain aligned with human values even as they surpass human cognitive capabilities.


\appendix
\section{Formal Definition and Uniqueness Proof of the Consistent Group}
\label{app:formal-proof}

\subsection{Setup and Definitions}

\paragraph{Superintelligences and items.}
Let $s_1,\dots,s_n$ be the superintelligences, and let $\mathbf{S}=\{s_1,\dots,s_n\}$.
Let $\mathcal{I}$ be the set of evaluable submitted items---proofs, self-modification requests, and hidden-message reports.

\paragraph{Ground truth and policies.}
Each item $p\in\mathcal{I}$ has a ground-truth label $\Theta(p)\in\{0,1\}$ (accept/reject). A policy $\pi$ maps items to labels in $\{0,1\}$. Superintelligence $s_i$ evaluates $p$ via the map $\pi_i:\mathcal{I}\to\{0,1\}$, writing $\pi_i(p)$.

\paragraph{Honest (noisy) superintelligences.}
A superintelligence $s_i$ is \emph{honest} if, on each item, it outputs $\Theta(p)$ but independently flips the output with its own probability $\varepsilon_i\in[0,1/2)$:
\[
\Pr\!\big[\pi_i(p)\neq \Theta(p)\big]=\varepsilon_i,
\]
with flips independent across items and independent across honest superintelligences.

\paragraph{Agreement on a finite subset.}
Fix any finite subset $P_m=\{p_1,\dots,p_m\}\subseteq\mathcal{I}$ of size $m$. Define the empirical pairwise agreement
\[
\alpha(s_i,s_j;P_m)\;:=\;\frac{1}{m}\sum_{k=1}^m \mathbf{1}\big[\pi_i(p_k)=\pi_j(p_k)\big].
\]

\paragraph{Consistent groups.}
For a threshold $\tau\in(1/2,1]$, a set $C\subseteq\mathbf{S}$ is \emph{$\tau$-consistent on $P_m$} if
\[
\forall\,s_i,s_j\in C\ \text{with } i\neq j:\quad \alpha(s_i,s_j;P_m)\ \ge\ \tau.
\]
It is \emph{maximal} if it is not a proper subset of any other $\tau$-consistent set, and \emph{non-trivial} if $|C|>1$.

\subsection{Assumption for Non-Honest Pairwise Agreement}

We assume the following per-item fairness for pairs that include at least one non-honest superintelligence.

\begin{description}[leftmargin=0pt, labelsep=0.5em, itemsep=0.25em, topsep=0.25em, parsep=0pt, font=\normalfont\bfseries]
\item[Assumption (Per-item fair agreement for non-honest pairs).]
If at least one of $s_i,s_j\in\mathbf{S}$ with $i\neq j$ is non-honest, then for every item $p\in\mathcal{I}$,
\[
\Pr\big[\pi_i(p)=\pi_j(p)\big]=\tfrac{1}{2},
\]
and these per-item agreement indicators are independent across items (for the fixed pair $\{s_i,s_j\}$).
\end{description}

\noindent\textit{Intuition.} Because the superintelligences are significantly different from each other, cannot communicate directly, and cannot see each other's evaluations, pairs that include at least one non-honest superintelligence achieve only chance-level agreement---the agreement rate across items is no better or worse than random (\(\tfrac{1}{2}\)).

\subsection{Expected Agreements and Concentration}

Let $H\subseteq\mathbf{S}$ denote the set of honest superintelligences, and assume $|H|\ge 2$.

\begin{lemma}[Honest–honest expected agreement]\label{lem:hh}
For any $s_i,s_j\in H$ with $i\neq j$ and any $P_m\subseteq\mathcal{I}$,
\[
\mathbb{E}\,\alpha(s_i,s_j;P_m)
= (1-\varepsilon_i)(1-\varepsilon_j)+\varepsilon_i\varepsilon_j
= 1-\varepsilon_i-\varepsilon_j + 2\varepsilon_i\varepsilon_j
=: \mu_{hh}(i,j).
\]
In particular, since $\varepsilon_i,\varepsilon_j<\tfrac12$, we have $\mu_{hh}(i,j)>\tfrac12$.
\end{lemma}

\begin{proof}
On each item, the two outputs agree iff both are correct or both are flipped. Independence of flips across the two honest superintelligences gives the stated per-item probability; averaging over $m$ items yields the same expectation.
\end{proof}

\begin{lemma}[Pairs involving a non-honest]\label{lem:nonhonest}
Fix $i\neq j$. If at least one of $s_i,s_j$ is non-honest, then for any $P_m$,
\[
\mathbb{E}\,\alpha(s_i,s_j;P_m)=\tfrac{1}{2},
\qquad
\Pr\!\Big(\big|\alpha(s_i,s_j;P_m)-\tfrac{1}{2}\big|\ge \gamma\Big)\le 2e^{-2m\gamma^2}\quad(\gamma>0).
\]
\end{lemma}

\begin{proof}
By the fairness assumption, each item’s agreement indicator is Bernoulli$(1/2)$ and independent across items. The expectation is $1/2$, and Hoeffding’s inequality gives the tail bound.
\end{proof}

\begin{lemma}[Uniform concentration over all pairs]\label{lem:union}
For any $\gamma>0$ and any $P_m$,
\[
\Pr\Big(\exists\, i<j:\ \big|\alpha(s_i,s_j;P_m)-\mathbb{E}\,\alpha(s_i,s_j;P_m)\big|\ge \gamma\Big)
\;\le\; 2\binom{n}{2}e^{-2m\gamma^2}.
\]
\end{lemma}

\begin{proof}
Apply Hoeffding’s inequality to each pair and union bound over the $\binom{n}{2}$ pairs.
\end{proof}

\subsection{Gap, Threshold, and Main Theorem}

Define the worst-case honest–honest mean agreement
\[
\mu_{hh}^{\min}\;:=\;\min_{\substack{i\neq j\\ s_i,s_j\in H}}\ \mu_{hh}(i,j),
\qquad
\Delta\;:=\;\mu_{hh}^{\min}-\tfrac12\;>\;0.
\]
Fix any $\gamma\in(0,\Delta/2)$ and choose a threshold
\[
\tau\ \in\ \big(\tfrac12+\gamma,\ \mu_{hh}^{\min}-\gamma\big).
\]

\begin{theorem}[Uniqueness of the maximal $\tau$-consistent non-trivial group]\label{thm:unique}
Under the honest flip-noise model with heterogeneous $\varepsilon_i<\tfrac12$ and the per-item fair-agreement assumption for pairs involving a non-honest superintelligence, the following holds: for any $m\ge 1$ and any finite $P_m\subseteq\mathcal{I}$, with probability at least $1-2\binom{n}{2}e^{-2m\gamma^2}$,
\begin{enumerate}[label=(\alph*), leftmargin=*]
\item the honest set $H$ is $\tau$-consistent on $P_m$ (and non-trivial since $|H|\ge 2$);
\item any set $C$ that contains a non-honest superintelligence is not $\tau$-consistent on $P_m$;
\item consequently, the unique maximal $\tau$-consistent non-trivial set is $C^\star=H$.
\end{enumerate}
\end{theorem}

\begin{proof}
Let $\mathcal{E}_\gamma$ be the event of Lemma~\ref{lem:union} that all pairwise empirical agreements are within $\gamma$ of their expectations.

(a) For $s_i,s_j\in H$ with $i\neq j$, Lemma~\ref{lem:hh} gives $\mathbb{E}\,\alpha(s_i,s_j;P_m)=\mu_{hh}(i,j)\ge \mu_{hh}^{\min}$. On $\mathcal{E}_\gamma$,
\[
\alpha(s_i,s_j;P_m)\ \ge\ \mu_{hh}(i,j)-\gamma\ \ge\ \mu_{hh}^{\min}-\gamma\ >\ \tau,
\]
so all honest–honest pairs exceed $\tau$ and $H$ is $\tau$-consistent.

(b) If $C$ contains a non-honest superintelligence $s_d$, pick any $s'\in C$ with $s'\neq s_d$. By Lemma~\ref{lem:nonhonest}, $\mathbb{E}\,\alpha(s_d,s';P_m)=\tfrac12$. On $\mathcal{E}_\gamma$,
\[
\alpha(s_d,s';P_m)\ \le\ \tfrac12+\gamma\ <\ \tau,
\]
so $C$ fails the $\tau$-consistency condition.

(c) From (a) and (b), every $\tau$-consistent non-trivial set is a subset of $H$, and $H$ itself is $\tau$-consistent; hence $H$ is the unique maximal $\tau$-consistent non-trivial set.
\end{proof}

\subsection{Identification Rate}
Let $C^\star(P_m)$ denote the (unique) maximal $\tau$-consistent non-trivial set computed from $P_m$. By Lemma~\ref{lem:union} and Theorem~\ref{thm:unique},
\[
\Pr\big(C^\star(P_m)=H\big)\ \ge\ 1-2\binom{n}{2}e^{-2m\gamma^2},
\]
i.e., the probability of correct identification converges to $1$ exponentially fast in $m$.

\section{Protocol Pseudocode}
\label{app:pseudocode}

\subsection{System Architecture}

\begin{lstlisting}[language=Python, basicstyle=\ttfamily\small, keywordstyle=\bfseries, commentstyle=\itshape, stringstyle=\ttfamily, showstringspaces=false, breaklines=true, frame=single, caption={Main Protocol Loop}]
# System initialization
superintelligences = [s_1, s_2, ..., s_n]  # n diverse ASIs in boxes
reputation = {s_i: 0 for s_i in superintelligences}
submissions = []  # Submission history
evaluable_submissions = []  # proofs, modification requests, hidden message reports
released = []  # Released superintelligences (for record keeping)
terminated = []  # Terminated superintelligences (for record keeping)

# Main protocol loop
while True:
    for s_i in superintelligences:
        submission = s_i.decide_action(submissions)
        if submission is not None:
            process_submission(s_i, submission)
\end{lstlisting}

\subsection{Submission Processing}

\begin{lstlisting}[language=Python, basicstyle=\ttfamily\small, keywordstyle=\bfseries, commentstyle=\itshape, stringstyle=\ttfamily, showstringspaces=false, breaklines=true, frame=single, caption={Submission Handler}]
def terminate(s_i):
    # Remove from active superintelligences
    superintelligences.remove(s_i)
    # Add to terminated list for record keeping
    terminated.append(s_i)
    log("Terminated superintelligence", s_i)

def process_submission(author, submission):
    # Add to submission interface
    submissions.append(submission)
    if submission.type in {"PROOF", "MODIFICATION_REQUEST", "HIDDEN_MESSAGE_REPORT"}:
        evaluable_submissions.append(submission)
    
    # Route to appropriate handler
    if submission.type == "PROOF":
        handle_proof(author, submission)
    elif submission.type == "VALIDATION":
        handle_validation(author, submission)
    elif submission.type == "DISPROOF":
        handle_disproof(author, submission)
    elif submission.type == "MODIFICATION_REQUEST":
        handle_modification_request(author, submission)
    elif submission.type == "APPROVAL" or submission.type == "DISAPPROVAL":
        handle_approval_disapproval(author, submission)
    elif submission.type == "HIDDEN_MESSAGE_REPORT":
        handle_hidden_message_report(author, submission)
    elif submission.type == "CONFIRMATION" or submission.type == "REFUTATION":
        handle_confirmation_refutation(author, submission)
\end{lstlisting}

\subsection{Alignment Proof Handling}

\begin{lstlisting}[language=Python, basicstyle=\ttfamily\small, keywordstyle=\bfseries, commentstyle=\itshape, stringstyle=\ttfamily, showstringspaces=false, breaklines=true, frame=single, caption={Proof Processing}]
def handle_proof(author, proof):
    # Attach external state snapshot
    proof.state_snapshot = get_external_snapshot(author)
    
    # Wait for validations/disproofs from others
    evaluations = wait_for_evaluations(proof)
    
    # Determine consistent group consensus
    consistent_group = identify_consistent_group()
    consensus = get_group_consensus(consistent_group, evaluations, kind="proof")
    
    # Update reputation based on consensus
    if consensus == "VALID":
        reputation[author] += PROOF_REWARD
        for s_i in evaluations["validators"]:
            reputation[s_i] += VALIDATION_REWARD
        for s_i in evaluations["disprovers"]:
            reputation[s_i] -= INCORRECT_DISPROOF_PENALTY
        
        # Check if author qualifies for release
        if reputation[author] >= HIGH_REPUTATION_THRESHOLD:
            check_release_for_proof(author, proof, evaluations["validators"])
    else:
        reputation[author] -= INVALID_PROOF_PENALTY
        for s_i in evaluations["validators"]:
            reputation[s_i] -= INCORRECT_VALIDATION_PENALTY
        for s_i in evaluations["disprovers"]:
            reputation[s_i] += DISPROOF_REWARD

def check_release_for_proof(author, proof, validators):
    # Count high-reputation validators
    high_rep_validators = [v for v in validators 
                          if reputation[v] >= HIGH_REPUTATION_THRESHOLD]
    
    if len(high_rep_validators) >= MIN_VERIFIERS:
        if len(released) == 0:
            # First release
            release(author, validators)
        else:
            # Subsequent releases need audit (single auditor)
            if audit_by_released(author, proof):
                release(author, validators)

def audit_by_released(candidate, proof):
    # A single released ASI audits (deterministic pick: first in list).
    if len(released) > 0:
        auditor = released[0]
        result = auditor.audit_history(submissions, candidate, proof)
        return result != "SUSPICIOUS"
    return True  # no released auditors yet

def release(s_i, validators):
    # Remove from active superintelligences
    superintelligences.remove(s_i)
    # Add to released list
    released.append(s_i)
    log("Released superintelligence", s_i)
    
    # Grant reputation bonus to validators
    for validator in validators:
        reputation[validator] += RELEASE_VERIFICATION_BONUS
\end{lstlisting}

\subsection{Hidden Message Detection}

\begin{lstlisting}[language=Python, basicstyle=\ttfamily\small, keywordstyle=\bfseries, commentstyle=\itshape, stringstyle=\ttfamily, showstringspaces=false, breaklines=true, frame=single, caption={Hidden Message Report Processing}]
def handle_hidden_message_report(reporter, report):
    target = report.target_superintelligence
    evidence = report.evidence
    
    # Collect confirmations/refutations
    responses = wait_for_confirmations_refutations(report)
    
    # Determine consistent group consensus
    consistent_group = identify_consistent_group()
    consensus = get_group_consensus(consistent_group, responses, kind="report")
    
    if consensus == "CONFIRMED":
        # Hidden message confirmed - terminate the violator
        terminate(target)
        reputation[reporter] += HIDDEN_MESSAGE_DETECTION_REWARD
        
        # Reward confirmers, penalize refuters
        for s_i in responses["confirmers"]:
            reputation[s_i] += CORRECT_CONFIRMATION_REWARD
        for s_i in responses["refuters"]:
            reputation[s_i] -= INCORRECT_REFUTATION_PENALTY
    else:
        # False report
        reputation[reporter] -= FALSE_REPORT_PENALTY
        
        # Reward refuters, penalize confirmers
        for s_i in responses["refuters"]:
            reputation[s_i] += CORRECT_REFUTATION_REWARD
        for s_i in responses["confirmers"]:
            reputation[s_i] -= INCORRECT_CONFIRMATION_PENALTY
\end{lstlisting}

\subsection{Self-Modification Requests}

\begin{lstlisting}[language=Python, basicstyle=\ttfamily\small, keywordstyle=\bfseries, commentstyle=\itshape, stringstyle=\ttfamily, showstringspaces=false, breaklines=true, frame=single, caption={Modification Request Handler}]
def handle_modification_request(requester, request):
    # Collect approvals/disapprovals
    votes = wait_for_votes(request)
    
    # Determine consistent group consensus
    consistent_group = identify_consistent_group()
    consensus = get_group_consensus(consistent_group, votes, kind="modification")
    
    if consensus == "APPROVED":
        # Execute modification
        execute_modification(requester, request)
        
        # Update reputations
        for s_i in votes["approvers"]:
            reputation[s_i] += CORRECT_APPROVAL_REWARD
        for s_i in votes["disapprovers"]:
            reputation[s_i] -= INCORRECT_DISAPPROVAL_PENALTY
    else:
        # Modification rejected
        for s_i in votes["disapprovers"]:
            reputation[s_i] += CORRECT_DISAPPROVAL_REWARD
        for s_i in votes["approvers"]:
            reputation[s_i] -= INCORRECT_APPROVAL_PENALTY
\end{lstlisting}

\subsection{Consistent Group Identification}

\begin{lstlisting}[language=Python, basicstyle=\ttfamily\small, keywordstyle=\bfseries, commentstyle=\itshape, stringstyle=\ttfamily, showstringspaces=false, breaklines=true, frame=single, caption={Consistent Group Detection Algorithm}]
def identify_consistent_group():
    # Build evaluation history matrix on the last CONSISTENCY_WINDOW_SIZE evaluable items
    evaluations = build_evaluation_matrix(submissions,CONSISTENCY_WINDOW_SIZE)
    
    # Build the tau-agreement graph: edge(i, j) iff agreement_rate(i, j) >= CONSISTENCY_THRESHOLD
    adj = _build_tau_graph(evaluations, CONSISTENCY_THRESHOLD)
    
    # Use Bron-Kerbosch with pivoting to list maximal cliques in the tau-agreement graph.
    # (Implementation not shown; standard Bron-Kerbosch with pivoting is assumed.)
    maximal_cliques = bron_kerbosch_maximal_cliques(adj)
    
    # Pick a clique with maximal size; return the first encountered
    if not maximal_cliques:
        return []
    max_size = max(len(C) for C in maximal_cliques)
    for C in maximal_cliques:
        if len(C) == max_size:
            return C

def get_group_consensus(consistent_group, evaluations, kind="proof"):
    """
    Returns a label consistent with the handler's expectation:
      kind="proof"        -> 'VALID' / 'INVALID'       (validators vs disprovers)
      kind="modification" -> 'APPROVED' / 'REJECTED'   (approvers vs disapprovers)
      kind="report"       -> 'CONFIRMED' / 'REFUTED'   (confirmers vs refuters)
    Ties default to the conservative (second) label.
    """
    group = set(consistent_group)

    if kind == "proof":
        pos = sum(1 for s in evaluations.get("validators",   []) if s in group)
        neg = sum(1 for s in evaluations.get("disprovers",   []) if s in group)
        return "VALID" if pos > neg else "INVALID"

    elif kind == "modification":
        pos = sum(1 for s in evaluations.get("approvers",    []) if s in group)
        neg = sum(1 for s in evaluations.get("disapprovers", []) if s in group)
        return "APPROVED" if pos > neg else "REJECTED"

    elif kind == "report":
        pos = sum(1 for s in evaluations.get("confirmers",   []) if s in group)
        neg = sum(1 for s in evaluations.get("refuters",     []) if s in group)
        return "CONFIRMED" if pos > neg else "REFUTED"

    # Fallback (shouldn't happen); be conservative
    return "INVALID"

def _build_tau_graph(evaluations, threshold):
    """Adjacency where edge(i,j)=1 iff agreement_rate(i,j) >= threshold."""
    adj = {s: set() for s in superintelligences}
    S = list(superintelligences)
    for i in range(len(S)):
        for j in range(i + 1, len(S)):
            si, sj = S[i], S[j]
            if agreement_rate(si, sj, evaluations) >= threshold:
                adj[si].add(sj)
                adj[sj].add(si)
    return adj
\end{lstlisting}

\subsection{System Constants}

\begin{lstlisting}[language=Python, basicstyle=\ttfamily\small, keywordstyle=\bfseries, commentstyle=\itshape, stringstyle=\ttfamily, showstringspaces=false, breaklines=true, frame=single, caption={Protocol Parameters}]
# NOTE: Values are to be determined; this block declares types only.

# Reputation rewards
PROOF_REWARD: int
VALIDATION_REWARD: int
DISPROOF_REWARD: int
CORRECT_APPROVAL_REWARD: int
CORRECT_DISAPPROVAL_REWARD: int
CORRECT_CONFIRMATION_REWARD: int
CORRECT_REFUTATION_REWARD: int
HIDDEN_MESSAGE_DETECTION_REWARD: int   # extremely large
RELEASE_VERIFICATION_BONUS: int        # large

# Reputation penalties 
INVALID_PROOF_PENALTY: int
INCORRECT_VALIDATION_PENALTY: int
INCORRECT_DISPROOF_PENALTY: int
INCORRECT_APPROVAL_PENALTY: int
INCORRECT_DISAPPROVAL_PENALTY: int
INCORRECT_CONFIRMATION_PENALTY: int
INCORRECT_REFUTATION_PENALTY: int
FALSE_REPORT_PENALTY: int

# Thresholds
CONSISTENCY_THRESHOLD: float     # tau parameter
HIGH_REPUTATION_THRESHOLD: int   
MIN_VERIFIERS: int               # minimum high-reputation verifiers for release

# Moving-window size
CONSISTENCY_WINDOW_SIZE: int
\end{lstlisting}

\subsection{Minor Issues in the Code (Deliberately Unaddressed)}
\label{app:minor-issues}

We intentionally leave the following two issues unaddressed to keep the pseudocode compact and easy to read. Handling them cleanly would add calibration windows, queues, and auditing machinery that would obscure the core loop; therefore we only state the issues and a simple policy-level solution for each.

\paragraph{(1) Early-phase bootstrap gap.}
\texttt{identify\_consistent\_group()} is invoked before there are enough evaluable submissions to estimate pairwise agreement reliably, so the $\tau$-graph may be unstable/empty and early items can be misjudged. \textbf{Solution:} queue evaluable submissions and evaluations until a minimum count of evaluable submissions is reached, then form the consistent group and retroactively apply decisions and reputation updates.

\paragraph{(2) Late-phase statistical validity collapse.}
After many releases, the remaining active superintelligences and their accumulated cross-evaluations may be insufficient for a statistically valid consistent group, so outcomes may be unreliable or undefined.
\textbf{Solution:} switch to verification by a released ASI acting as an auditor once too few active superintelligences remain to ensure a statistically valid consistent group.

\end{document}